\documentclass[journal]{IEEEtran}
\usepackage{amsmath,amsfonts}
\usepackage{algorithmic}
\makeatletter
\newcommand{\removelatexerror}{\let\@latex@error\@gobble}
\makeatother
\usepackage{algorithm}
\usepackage{array}
\usepackage[caption=false,font=normalsize,labelfont=sf,textfont=sf]{subfig}
\usepackage{textcomp}
\usepackage{stfloats,amsthm}
\usepackage{url}
\usepackage{verbatim}
\usepackage{graphicx}
\usepackage{cite,color,xcolor}
\usepackage{booktabs}
\usepackage{multirow} 
\usepackage[backref]{hyperref} 
\usepackage[numbers,sort&compress,square]{natbib}
\newtheorem{theorem}{Theorem}
\newtheorem{remark}{Remark}
\usepackage[utf8]{inputenc}
 \newcommand{\h}[1]{\mathbf{#1}}
\usepackage{cleveref}
\usepackage{etoolbox}
\makeatletter
\patchcmd{\@makecaption}
  {\scshape}
  {}
  {}
  {}
\makeatother
\begin{document}

\title{Robust retrieval of material chemical states in X-ray microspectroscopy }

\author{Ting Wang, Xiaotong Wu, Jizhou Li, Chao Wang 
\thanks{Ting Wang, Xiaotong Wu, and Chao Wang are with the Department of Statistics and Data Science, Southern University of Science and Technology, Shenzhen 518055, China. Chao Wang is also with National Centre for Applied Mathematics Shenzhen, Shenzhen 518055, China.}
\thanks{Jizhou Li is with the School of Data Science, City University of Hong Kong, Kowloon, Hong Kong.}
\thanks{Chao Wang was partially supported
by the Natural Science Foundation of China (No. 12201286), HKRGC Grant No.CityU11301120, and the Shenzhen
Fundamental Research Program  JCYJ20220818100602005. (Corresponding author: Chao Wang, email: wangc6@sustech.edu.cn) }
}



\maketitle

\begin{abstract}
    X-ray microspectroscopic techniques are essential for studying morphological and chemical changes in materials, providing high-resolution structural and spectroscopic information. However, its practical data analysis for reliably retrieving the chemical states remains a major obstacle to accelerating the fundamental understanding of materials in many research fields. In this work, we propose a novel data formulation model for X-ray microspectroscopy and develop a dedicated unmixing framework to solve this problem, which is robust to noise and spectral variability. Moreover, this framework is not limited to the analysis of two-state material chemistry,
making it an effective alternative to conventional and widely-used methods. In addition, 
an alternative directional multiplier method  with provable convergence is applied to obtain the solution efficiently.
Our framework can accurately identify and characterize chemical states in complex and heterogeneous samples, even under challenging conditions such as low signal-to-noise ratios and overlapping spectral features. 
Extensive experimental results on simulated and real datasets demonstrate its effectiveness and reliability. 
\end{abstract}


\section{Introduction}
\IEEEPARstart{X}{-ray} absorption spectroscopy (XAS) is a scientific technique that utilizes X-rays to investigate the electronic and structural properties of materials. However, the spatial resolution of XAS is typically limited to the micron or sub-micron scale, which poses a challenge when studying materials with complex or heterogeneous structures. In recent years, spectroscopic full-field transmission X-ray microscopy (TXM) has emerged as a novel tool for nanoscale chemical imaging, with great potential in various multidisciplinary fields~\cite{osti_1382491,wang2014operando}. By imaging at energy points across the absorption edge of the element of interest, TXM offers both high spatial resolution and chemical-specific information. Sub-50-nm resolution X-ray absorption near-edge structure (XANES) spectroscopy is routinely achieved with TXM-XANES, mainly operating in the hard X-ray regime (5 to 12 keV)~\cite{boesenberg2013mesoscale,yang2019simultaneously,zhang2017finding}.  
Its application areas include materials science, physics, chemistry, and biology. For instance, it can be used for chemical mapping in battery studies~\cite{xu2017situ,jiang2020machine} and mesoscale degradation~\cite{qian2021understanding}.

In TXM-XANES, the intensity change of each pixel is scrutinized to generate XANES spectra that are matched against reference compounds. Some common techniques, including the edge-50 or linear combination fitting (LCF)~\cite{osti_1382491},  are used to fit the spectra, then a two-dimensional colormap is constructed to display the chemical phase combination of each pixel. The XANES edge-50 point (energy at 0.5 spectrum position), which measures the absorption spectra of materials within the energy range of 5 keV to 12 keV, is a specific type of XANES spectrometer. The utilization of the edge-50 XANES technique has been progressively examined for characterizing the chemical composition and structure of environmental material~\cite{nelson2011three}.
On the other hand,  \cite{newville2014fundamentals} proposed using LCF to determine the phase composition of a chemical sample from normalized XANES spectra. The XANES image at each pixel represents a spectrum at a particular location, which can be fitted with reference spectra to produce spatially resolved chemical state information. This technique significantly simplifies the processing and analysis of XANES spectra using LCF.  These traditional methods have been extensively used in the literature~\cite{prietzel2011sulfur,gustafsson2020probabilistic, hesterberg2017speciation}.

Although traditional methods are widely applicable, these rely on the high quality of XANES imaging. In this case, a relatively slow acquisition process, with the recording of hundreds or thousands of energy points, is needed to achieve sufficient energy resolution. Fast TXM-XANES imaging is crucial for reliably solving morphological chemical phase transitions, as in 3D battery material research. To increase the speed of TXM-XANES imaging, energy points are reduced, or X-ray exposure time is minimized, which is more favorable for radiation-sensitive samples, similar to low-dose medical X-ray imaging applications.
However, excessively short exposure times can result in measurements with strong noise~\cite{li2022subspace}. 
Furthermore, during the process of acquiring XANES data, there are many variations in the X-ray exposure conditions, and inherent material properties, which contribute to the variability of XANES spectra~\cite{anzures2020xanes}. 
In the face of strong noise and spectral variability, the edge-50 and LCT methods fail to obtain a reasonable interpretation of elemental and chemical information. Despite efforts to optimize microscope hardware, the physical limitations of the TXM imaging system remain difficult to overcome. To address this obstacle, computational algorithm development is inevitable for improving downstream analysis through fitting results.

  Spectral unmixing methods~\cite{keshava2002spectral} have numerous applications in image science, including remote sensing~\cite{ma2013signal,li2020superpixel,wang2022adaptive}, optical microscopy~\cite{tzoumas2017spectral}, and X-ray imaging~\cite{ayhan2015use,yangdai2017spectral,rossouw2015multicomponent}. 
 The unmixing technique aims to decompose a spectrum of mixed pixels into a set of distinct spectral signatures, known as endmembers, along with their corresponding fractional abundances ~\cite{heylen2014review,bioucas2012hyperspectral}. By utilizing spectral unmixing in X-ray microspectroscopy, the chemical states of materials can be directly obtained bypassing the fitting process. 
 Various regularizations have been developed in spectral unmixing methods to utilize the prior information on the abundance map against noise. In addition, in the face of spectral variability, many model formulations have been proposed in the unmixing problems~\cite{drumetz2016blind,hong2018augmented,borsoi2021spectral,drumetz2020spectral,azar2021linear}.  
 The principle underlying the LCF method is essentially spectral unmixing~\cite{keshava2002spectral}, whereby the mixture is analyzed by determining the contribution of the reference spectra. However, it is sensitive to noise and limited in handling problems with spectral variability. 
 
The XANES unmixing task involving spectral variability can be formulated as an optimization model with some proper priors. We employ two regularization techniques to achieve this: the total-variation (TV) regularizer and the Plug-and-Play (PnP) prior. The TV regularizer is applied to the reconstructed image to incorporate spatial and spectral information through pixel connections in the unmixing process~\cite{iordache2012total,cruz2021extended}. On the other hand, the PnP technique utilizes state-of-the-art denoisers to tackle linear inverse problems in various hyperspectral image processing tasks~\cite{lin2019hyperspectral,gong2020learning,wang2020learning,zhao2021plug,chen2023integration}. 
The main contributions of this paper are summarized as follows:
\begin{itemize}
    \item We present a novel and robust framework for X-ray XANES imaging, which incorporates various realistic factors that affect the spectra, such as noise and spectral variability.
    \item The convergence of our proposed framework with TV regularization is theoretically analyzed, demonstrating its effectiveness and applicability.
    \item Our proposed framework is evaluated extensively using both quantitative and qualitative methods on synthetic and real datasets. The results indicate that our proposed methods surpass the state-of-the-art. Especially our framework with a PnP prior achieves the best performance. 
\end{itemize}

The rest of the paper is organized as follows. \Cref{sec:related}  briefly describes two related works.
In \Cref{sec:proposed}, we propose a novel data formulation model for the material chemical states retrieval in X-ray microspectroscopy and the corresponding algorithms to solve it. The convergence analysis is shown in \Cref{sec:convergence}.  \Cref{sec:results} presents the experimental results and some subsequent discussions. Finally, \Cref{sec:concl} provides a summary and future perspectives.
\section{Related Work}\label{sec:related}
This section will present a concise overview of two conventional methods to address this problem.  Furthermore, the form of the edge-50 and LCF serves as the baseline for our proposed approach.
\subsection{Edge-50}
The edge-50 point, representing the energy at the 0.5 spectrum position, is a crucial parameter in studying chemical state changes during battery cycling. It is often used to determine the phase map in a TXM-XANES system~\cite{andrews2010nanoscale, liu2011phase}.

After preprocessing and the post-edge and pre-edge region normalization~\cite{osti_1382491} of XANES images, a XANES spectrum is constructed for each pixel by plotting normalized absorption versus energy.
We can then obtain the ratio of chemical materials for each pixel by comparing the energies at the 0.5 spectrum position from the referenced spectra.

\subsection{Linear combination fitting (LCF)}
Besides the edge-50, LCF is another technique used to extract information about the electronic structure of elements in a complex sample from XANES spectra~\cite{newville2014fundamentals}. By using a least-squares approach, this technique decomposes the spectrum into a linear combination of reference spectra from known compounds. In fact, LCF is the classical linear mixing model (LMM) in the hyperspectral imaging unmixing field~\cite{keshava2002spectral, drumetz2016blind}.

Here the observed XANES image is represented by $\h Y = [\h y_1, \h y_2, \dots, \h y_N]\in \mathbb{R}^{T\times N}$, where each column vector is obtained by lexicographically ordering the TXM image with size $N=M\times K$, and $T$ is the number of energy points.
The LCF model generates the noisy measurements $\h Y$ from the chemical phase map $\mathbf{X}=[\h x_1, \h x_2, \dots, \h x_n]\in \mathbb{R}^{L\times N}$, pixel-wisely.
\begin{equation}
    \label{eq:LMM_pixel}
    \h y_k = \h A \h x_k + \h r_k, \ k = 1, \dots, N,
\end{equation}
where $\h A\in \mathbb{R}^{T\times L}$ is the dictionary, representing the reference materials spectra in the XANES images, $L$ is the number of materials, and $\h r_k$ is assumed to follow Gaussian distribution. In other words, \eqref{eq:LMM_pixel} can be rewritten in a matrix form:
\begin{equation}
\label{eq:LMM}
    \h Y = \h A \h X +\h R,
 \end{equation}
by considering all the pixels. Since the weight of a linear combination or chemical phase map is non-negative and sum-to-one, the estimation of the chemical phase map can be obtained by solving the following optimization problem:
\begin{equation}
\label{eq:fitting}
\begin{aligned}
\underset{\h X} \min \quad &\frac{1}{2}\|\h Y-\h A \h X\|_F^2\\
\text{s.t.}\quad &\h X\geq \mathbf{0}, \ \mathbf{1}^T \h X=\h 1.
\end{aligned}
\end{equation}
The chemical phase map $\h X$ in a model \eqref{eq:fitting} can be obtained using a non-negative constrained least squares algorithm, such as~\cite{heinz2001fully}.

Although traditional methods can quickly determine the composition of unknown samples, these methods are not robust under strong noise settings. 
\section{Proposed  Robust Unmixing Framework}\label{sec:proposed}
This section presents two models for effectively unmixing XANES images from model-based and learning-based perspectives, respectively. Both of them are solved by ADMM-based algorithm.
\subsection{Model formulation}
\begin{figure*}[ht]
    \begin{center}
\includegraphics[width=0.9\textwidth]{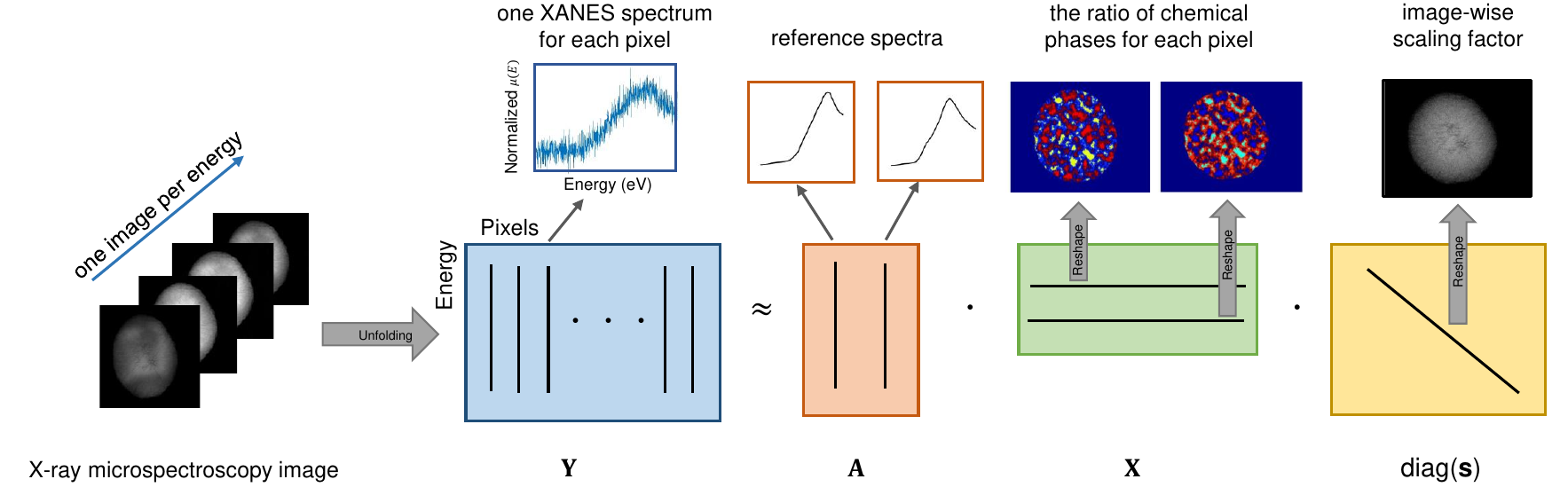} 
\end{center}
\caption{Framework of the proposed model for material chemical states retrieval in the form of unmixing for X-ray microscpectroscopy. The normalized XANES spectra from each pixel are unmixing to create a chemical phase map $\h X$, which also takes into account the scaling factor on the image. }
       \label{flowchart}
	\end{figure*}

The spectral variability induced by the various conditions can be effectively modeled by approximating the dictionary matrix of each pixel with the scaled version of the reference spectra. We propose an extended LCF model  by considering the variability in each pixel, i.e., \eqref{eq:LMM_pixel} can be changed to 
\begin{equation}
    \label{eq:ELMM_pixel}
        \h y_k = s_k\h A \h x_k + \h r_k, \ k = 1, \dots, N,
\end{equation}
where $s_k$ is a scalar in the $k$-the pixel. Similar to \eqref{eq:LMM}, we get the matrix form of \eqref{eq:ELMM_pixel} as 
\begin{equation}
    \label{eq:ELMM}
    \h Y = \h A \h X \mathrm{diag}(\h s) +\h R,
\end{equation}
where $ \mathrm{diag}(\h s)$ represents a diagonal matrix with its diagonal values $\h s = [s_1, s_2, \dots s_N]^T $ and $s_i \geq 0, \forall i\in 1, \dots, N$. Fig.  \ref{flowchart} gives the macro diagram of spectral unmixing for the XANES imaging.
With \eqref{eq:ELMM}, we get an optimization problem:
\begin{equation}
\label{eq:main_model}
\begin{split}
\underset{\h X,\h s} \min \quad &\frac{1}{2}\|\h Y-\h A \h X \mathrm{diag}(\h s)\|_F^2\\
\text{s.t.} \quad  &\h X\geq \mathbf{0}, \   \h s\geq \mathbf{0}, \ \mathbf{1}^T \h X=\h 1.
\end{split}
\end{equation}
Note that \eqref{eq:fitting} is equivalent to \eqref{eq:main_model} if we skip the sum-to-one constraint. However, combining $\h X$ and $\h s$ into a nonnegative least squares problem would lose some prior information on $\h X$ itself, especially when $\h X$ and $\h s$ are independent. In the following, we will utilize the prior information and propose a robust optimization framework under a low exposure time measurement. 
\subsection{Model-based approach: TV regularization}
TV regularization is a widely-used technique in image processing to promote sparsity in the gradient of the image~\cite{meiniel2018denoising,peng2021low}. Here the chemical phase map can be regarded as a group of images and have a piece-wise spatial correlation.  Hence, we first adapt TV regularization into \eqref{eq:main_model} and the proposed model can be expressed as follows:
\begin{equation}
\label{equ1}
\underset{\h X,\h s} \min \quad \frac{1}{2}\|\h Y-\h A \h X \mathrm{diag}(\h s)\|_F^2+\lambda \sum_{j=1}^L\|\nabla \h x_j\|_{1} + I_{\Omega_1}(\h X)+I_{\Omega_2}(\h s),
\end{equation}
where $\lambda$ is the regularization parameter,  $\h x_j$ is the $j$-th row in the chemical map $\h X$, and  the $L_1$ norm for vectors is denoted by $\|\cdot\|_1$. 
We define a discrete gradient operator,
$$
		\nabla \h z:= \left[\begin{matrix}
		\nabla_x\\
		\nabla_y
		\end{matrix} 
		\right]\h z, \forall \  \h z\in \mathbb{R}^N,
$$
where $\nabla_x,\nabla_y$ are the finite forward difference operator with a periodic boundary condition in the horizontal and vertical directions, respectively. 
Here $I_{\Omega}$ is the indicator function for the nonnegative value, i.e., 
\begin{equation}
I_\Omega(\h x)= 
    \begin{cases}
    0 & \h x \in  \Omega, \\
    +\infty & \text{otherwise. }
    \end{cases}
\end{equation}
In addition,  $\Omega_1 = \{ \h X | \h X\geq \mathbf{0} \text{ and }  \mathbf{1}^T \h X=\h 1 \}$ and $\Omega_2 = \{\h s |  \h s\geq \mathbf{0}\}$. 
After splitting the variables, the problem in \eqref{equ1} with auxiliary variables can be expressed as follows:
\begin{equation}
\label{eq:constraint}
\begin{split}
 \min_{\h X,\h s} \ &\frac{1}{2}\|\h Y-\h A \h M\|_F^2+\lambda\sum_{j=1}^L \|\h u_{j}\|_{1}
+ I_{\Omega_1}(\h W)+I_{\Omega_2}(\h t)\\
\text{s.t.} \ &\h M=\h X \mathrm{diag}(\h s), \h W= \h X, \h t=\h s\h,\\
&\h u_{j}=\nabla \h x_j, \text{ for } j = 1, \dots, L.
\end{split}
\end{equation}
Denote $\h U = [\h u_1, \h u_2, \dots, \h u_L]$, the corresponding augmented Lagrangian is:
\begin{equation}
\label{eq7}
\begin{split}
\mathcal{L}_1 & (\h X, \h s, \h M, \h U,\h W,\h t, \h F)\\
&=\frac{1}{2}\|\h Y-\h A \h M\|_F^2+\lambda\sum_{j=1}^L\|\h u_j\|_{1}+ I_{\Omega_1}(\h W)+I_{\Omega_2}(\h t)\\
&+\frac{\rho}{2}\|\h X \mathrm{diag}(\h s)-\h M+ \h C\|_F^2
-\frac{\rho}{2} \|\h C\|_F^2\\
&  +\frac{\rho}{2}\sum_{j=1}^L\left(\|\nabla \h x_j- \h u_{j}+\h d_j\|_F^2- \|\h d_j\|_F^2\right)\\
  &+ \frac{\rho}{2}\|\h X-\h W+\h E\|_F^2-\frac{\rho}{2} \|\h E\|_F^2\\
  &+\frac{\rho}{2}\|\h s-\h t+\h g\|_F^2-\frac{\rho}{2} \|\h g\|_F^2,
\end{split}
\end{equation}
where $\h C$, $\h D = [\h d_1, \h d_2, \dots, \h d_L]$, $\h E$,  $\h g$ are dual variables and $\lambda$,  $\rho$ are positive parameters.  Denote $\h F = [\h C; \h D; \h E;  \mathrm{diag}(\h g) ]$. We apply ADMM as the following scheme,
\begin{equation}
\label{eq8}
 \left \{
\begin{aligned}
 &\h X^{k+1} =\underset{\h X}{\arg\min} \ \mathcal{L}_1  (\h X, \h s^k, \h M^k, \h U^k,\h W^k, \h t^k,\h F^k),\\
 &\h s^{k+1} =\underset{\h s}{\arg \min} \ \mathcal{L}_1  (\h X^{k+1}, \h s, \h M^k, \h U^k,\h W^k,\h t^k,\h F^k),\\
 &\h M^{k+1} =\underset{\h M}{\arg \min} \ \mathcal{L}_1  (\h X^{k+1}, \h s^{k+1}, \h M, \h U^k,\h W^k,\h t^k,\h F^k),\\
& \h U^{k+1} =\underset{\h U}{\arg \min} \ \mathcal{L}_1  (\h X^{k+1}, \h s^{k+1}, \h M^{k+1}, \h U,\h W^k,\h t^k,\h F^k),\\
& \h W^{k+1} =\underset{\h W}{\arg \min}  \ \mathcal{L}_1  (\h X^{k+1}, \h s^{k+1}, \h M^{k+1}, \h U^{k+1},\h W,\h t^k,\h F^k),\\
& \h t^{k+1} =\underset{\h t}{\arg \min}  \ \mathcal{L}_1  (\h X^{k+1}, \h s^{k+1}, \h M^{k+1}, \h U^{k+1},\h W^{k+1},\h t,\h F^k),\\
& \h C^{k+1}= \h C^k+\h X^{k+1}\mathrm{diag}(\h s^{k+1})-\h M^{k+1},\\
& \h d_j^{k+1} = \h d_j^k+\nabla \h x_j^{k+1}-\h u^{k+1}_j, \text{ for } j= 1,\dots, L,\\
& \h E^{k+1} =  \h E^k+\h X^{k+1}-\h W^{k+1},\\
& \h g^{k+1} =\h g^k+ \h s^{k+1}-\h t^{k+1},\\
& \h F^{k+1} = [\h C^{k+1}; \h D^{k+1}; \h E^{k+1};  \mathrm{diag}(\h g^{k+1})]. 
\end{aligned}\right.
\end{equation}
We then elaborate on how to solve the six subproblems in \eqref{eq8}, by taking the derivative of $\mathcal{L}_1$ with respect to $\h X$, we obtain a closed-form solution,
\begin{equation}
\label{eq9}
\begin{aligned}
\h X^{k+1}=&\large((\h M^k-\h C^k)(\mathrm{diag}(\h s^k)+\sum_{j=1}^L \nabla^T(\h u^k_j-\h d_j^k) \\
&+\h W^k -\h E^k)\large) \large(\mathrm{diag}(\h s^k)^2-\Delta+ \h I\large)^{-1},
\end{aligned}
\end{equation}
where $\Delta=-\nabla^T\nabla$ represents the Laplacian operator. To solve for \eqref{eq9}, we use the conjugate gradient (CG) descent iterations. The $\h s$-subproblem in \eqref{eq8} has a closed-form solution.
\begin{equation}
\h s^{k+1} =  ( (\h X^{k+1})^T\h X^{k+1}+\h I)^{-1}((\h X^{k+1})^T(\h M^k-\h C^k)+\h t^k-\h g^k).
\end{equation}
The $\h M$-subproblem in \eqref{eq8} also has a closed-form solution, i.e.,
\begin{equation}
\h M^{k+1}=(\h A^T \h A+\rho \h I)^{-1}(\h A^T \h Y+\rho \h X^{k+1}\mathrm{diag}(\h s^{k+1}) +\rho \h C^k).
\end{equation}
To obtain the $\h {U}$-subproblem, we take the derivative of $\mathcal{L}_1$ with respect to $\h {u}_j$.
\begin{equation}
\label{shrink}
\h u_j^{k+1}=\textbf{shrink}(\nabla \h x_j^{k+1}+\h d_j^k,\frac{\lambda}{\rho}),\text{for} \ j= 1,\dots, L,
\end{equation}
where $\textbf{shrink}(x,\lambda)=\text{sign}(x)\max(|x|-\lambda,0)$.
Here we impose the sum-to-one constraint by normalizing $\h X $ in each iteration, then the subproblem of $\h W$ and $\h t$ are only for the non-negativity-constrained projection operator. Algorithm \ref{alg1} summarizes the whole process for solving model \eqref{eq7}.
\begin{figure}[htbp]
\renewcommand{\algorithmicrequire}{\textbf{Input:}}
\renewcommand{\algorithmicensure}{\textbf{Output:}}
 \removelatexerror
\begin{algorithm}[H]
	\caption{The Framework for XANES Image Unmixing with TV Regularization and PnP Priors. }
   \label{alg1}
     \begin{algorithmic}[1]
    \REQUIRE {A XANES image $\h Y$, Dictionary $\h A$.}
   \ENSURE {Phase map $\h X$, Scaling factor $\h s$.}
	\STATE Initialize: $\h X$ and $\h s$ and choose parameter $\rho$ and $\lambda$.
    \WHILE{not converged or iterations are not reached}
\STATE $\h X \,\leftarrow  \begin{cases}
      \text{is updated by} \ \eqref{eq9}  \,\text{for TV },\\
      \text{is updated by}  \ \eqref{PnPX}  \, \text{for PnP},
     \end{cases}$
     \STATE{Normalize $\h X$ such that $\mathbf{1}^T \h X=\h 1, $ }
 \STATE $ \h s \,\,\;\leftarrow 
    (\h X^T \h X+\h I)^{-1} (\h X^T(\h M-\h C)
     + \h t-\h g),$
  \STATE $ \h M  \leftarrow (\h A^T \h A+\rho \h I)^{-1}(\h A^T \h Y+\rho \h X \mathrm{diag}(\h s) +\rho \h C)$,
  \STATE $  \h U \, \leftarrow \begin{cases}
            \textbf{shrink}(\nabla \h x_j+\h d_j,\frac{\lambda}{\rho}), & \text{for TV }\\
             \textbf{Denoiser}(\h x_j+\h d_j,\frac{\lambda}{\rho}), & \text{for PnP}
        \end{cases}$
 \STATE $ \h W  \leftarrow \max(\h X+\h E,\h 0),$
   \STATE $ \h t \,\,\; \leftarrow \max(\h s+\h g,\h 0),$
   \STATE $      \h C \,\leftarrow \h C+\h X \mathrm{diag}(\h s)-\h M,$
  \STATE $\h d_j \leftarrow \begin{cases}
         \h d_j+\nabla \h x_j-\h u_j, & \text{for TV}\\
           \h d_j+\h x_j-\h u_j,& \text{for PnP}
      \end{cases}$
 \STATE $ \h E \,\leftarrow \h E+\h X-\h W,$
  \STATE $  \h g \; \leftarrow \h g+\h s-\h t $.
  \ENDWHILE
\end{algorithmic}
\end{algorithm}
\end{figure}
\subsection{Learning-based approach: PnP Priors}
 Designing a powerful regularizer can be challenging, as complex regularizers often complicate optimization problems, making the entire process more difficult. Rather than using a handcrafted regularizer, we aim to leverage prior knowledge about the spectral characteristics of materials in the scene to achieve better regularized unmixing results. Our proposed method can be seamlessly integrated into existing optimization frameworks, making it a practical and efficient tool for XANES image unmixing. The optimization problem can be formulated as follows:
\begin{equation}
\label{eq:pnp1}
\underset{\h X,\h s} \min \quad \frac{1}{2}\|\h Y-\h A \h X \mathrm{diag}(\h s)\|_F^2+\lambda \sum_{j=1}^L\Phi(\h x_j)  + I_{\Omega_1}(\h X)+I_{\Omega_2}(\h s),
\end{equation}
where $\Phi(\h X)$ represents some regularization term enforcing prior knowledge of $\h X$. 

By introducing the auxiliary variables, the optimization problem of \eqref{eq:pnp1} can be written in the equivalent form:
\begin{equation}
\begin{split}
&\underset{\h X,\h s} \min \quad \frac{1}{2}\|\h Y-\h A \h M\|_F^2+ \lambda \sum_{j=1}^L \Phi(\h u_j) + I_{\Omega_1}(\h W)+I_{\Omega_2}(\h t)\\
&\text{s.t.} \quad \h M=\h X \mathrm{diag}(\h s), \h u_j=\h x_j, \h W= \h X, \h t=\h s.
\end{split}
\end{equation}
and the augmented  Lagrangian is as follows:
\begin{equation}
\label{eq:pnp3}
\begin{split}
\mathcal{L}_{2}& (\h X, \h s, \h M, \h U,\h W,\h t, \h F)\\
&=\frac{1}{2}\|\h Y-\h A \h M\|_F^2+\lambda\sum_{j=1}^L\Phi(\h u_j)\\
&+\frac{\rho}{2}\|\h X \mathrm{diag}(\h s)-\h M+\h C\|_F^2-\frac{\rho}{2} \|\h C\|_F^2\\
&+\frac{\rho}{2} \sum_{j=1}^L\|\h x_j-\h u_j+\h d_j\|_F^2-\frac{\rho}{2} \|\h d_j\|_F^2\\
  &+ \frac{\rho}{2}\|\h X-\h W+\h E\|_F^2-\frac{\rho}{2} \|\h E\|_F^2\\
  &+\frac{\rho}{2}\|\h s-\h t+\h g\|_F^2-\frac{\rho}{2} \|\h g\|_F^2,
\end{split}
\end{equation}
where $\h C$, $\h D$, $\h E$, $\h g$ are dual variables and $\lambda$, $\rho$ are positive parameters. The method for solving model \eqref{eq:pnp3} is similar to that of model \eqref{eq7}. Specifically, 
by taking the derivative of $\mathcal{L}_2$ with respect to $\h X$, the $\h X$-subproblem has a closed-form solution. 
\begin{equation}
\label{PnPX}
\begin{aligned}
 \h X^{k+1}=&((\h M^{k}-\h C^{k})\mathrm{diag}(\h s^{k})+\sum_{j=1}^L(\h u^k_j-\h d^k_j) \\
 &+\h W^k-\h E^k)(\mathrm{diag}(\h s^k)^2 + 2\h I)^{-1}.
\end{aligned}
\end{equation}
The $\h u_j$-subporblem is to solve a proximal 
 proximal operator as follows:
\begin{equation}
\label{eq:prox}
\h u_j^{k+1}= \arg \min_{\h u} \frac{\rho}{2}||\h u-\h x^{k+1}_j-\h d^k_j||_F^2 +\lambda \phi(\h u),
\end{equation}
which 
can be viewed as an image-denoising problem. The goal is to eliminate additive Gaussian noise with a standard deviation of $\sigma =\sqrt{ \lambda/\rho}$. We employ established and effective denoising operators in the PnP framework iterations, such as the conventional BM3D ~\cite{dabov2007image} or DnCNN~\cite{zhang2017beyond}, which utilizes deep learning. After acquiring the necessary denoising operators, we update the primal and dual variables in the ADMM process, following Algorithm \ref{alg1}.

\begin{remark}{Dictionary selection:}
The proposed algorithm can quickly and accurately extract the spectral signal from the XANES imaging data. However, the reference spectra are a critical component for achieving optimal performance. 
When the reference spectra are unknown, we use the conventional spectra extraction method, which is the vertex component analysis (VCA) ~\cite{nascimento2005vertex} as a baseline for dictionary identification. In the real data analysis in Section \ref{c44}, we demonstrate that using VCA with denoising results in more accurate reference spectra extraction, particularly in strong-noise environments.
\end{remark}

\section{Convergence analysis}\label{sec:convergence}
In this section, we demonstrate the convergence of Algorithm \ref{alg1} with TV regularization.
Under some assumptions, we prove an optimal solution to the problem \eqref{eq7} can be found by iterating the scheme \eqref{eq8}. 
Although this result is not entirely satisfactory, it provides some theoretical guarantees for the reliability of Algorithm \ref{alg1} with TV regularization.  

Assuming  $\mathcal{X}^{\star}=\{\h X^{\star},\h s^{\star},\h M^{\star},\h U^{\star},\h W^{\star},\h t^{\star},\h C^{\star},\h D^{\star},\h E^{\star},\h g^{\star}\}$ as the fixed point, then the Karush-Kuhn-Tucker (KKT) condition of \eqref{eq:constraint} can be summarized as follows: 
\begin{equation}
\label{eq10}
\begin{split}
\left \{
\begin{array}{llllll}
\h 0=&\h M^{\star}-\h X^{\star} \mathrm{diag}(\h s^{\star}),\\
\h 0=&\h u_j^{\star}-\nabla \h x_j^{\star}, \text{ for } j= 1,\dots, L, \\
\h 0=&\h W^{\star}-\h X^{\star},\\
\h 0=&\h t^{\star}-\h s^{\star},\\
\h 0=&\h C^{\star}\mathrm{diag}(\h s^{\star})+ \sum_{j=1}^L\nabla^T \h d_j^{\star}+\h E^{\star},\\
\h 0=&\h X^{\star T}\h C^{\star}+\rho \mathrm{diag}(\h g^{\star}),\\
\h 0=& \h A^T(\h Y-\h A\h M^{\star})-\rho \h C^{\star},\\
\h 0 \in & \lambda \partial \|\h u_j^{\star}\|_1-\rho \h d_j^{\star},  \text{ for } j= 1,\dots, L,\\
\langle\h E^{\star},&\h W-\h W^{\star}\rangle \leq 0 \quad \forall \ \h W \in \Omega_1, \\
\langle\h g^{\star},&\h t-\h t^{\star}\rangle \leq 0 \quad \forall \ \h t \in \Omega_2.\\
\end{array}
\right.
\end{split}
\end{equation}
\begin{theorem}
\label{thm1}
Let $\mathcal{X} ^k
$ be generated by Algorithm \ref{alg1}, if the successive differences of the multipliers $\h C^{k+1}-\h C^k$, $\h D^{k+1}-\h D^k$, $\h E^{k+1}-\h E^k$ and $\h g^{k+1}-\h g^k$ all converges to $\h 0$ as $k$ tends to $\infty$, and if $\{\h s^k\}$ is bounded, then there exists a subsequence $\mathcal{X}^{k_j}$
whose accumulation point satisfies the KKT condition of \eqref{eq:constraint}.
\end{theorem}
\begin{proof}
Since  $\lim_{k \rightarrow \infty} \h C^{k+1}-\h C^k=\h 0$ and $\lim_{k \rightarrow \infty} \h D^{k+1}-\h D^k=\h 0$, and the multiplier updates are given by \eqref{eq8}, it can be conclude that
\begin{equation}\label{eq:kkt_cond}
\begin{aligned}
\lim_{k \rightarrow \infty}& \h X^k\mathrm{diag}(\h s^k)-\h M^k=\h 0,\\
\lim_{k \rightarrow \infty}& \nabla \h x_j^k-\h u_j^k=\h 0, \text{ for } j= 1,\dots, L.
\end{aligned}
\end{equation}
Here $\h W^k$ is bounded owing to the constraint $\Omega_1$. Incorporating  $\lim_{k \rightarrow \infty} \h E^{k+1}-\h E^k=\h 0$, we get the boundedness of $\h X^k $. Hence there exists a bounded subsequence such that  $\lim_{j \rightarrow \infty} \h E^{k_j} =  \lim_{j \rightarrow \infty} \h X^{k_j} = \h X^\star $. It can be inferred that both $\h M^k, \h U ^k$, and $\h t$ are bounded from the same analysis. Therefore, the following system of equations holds:
\begin{equation}
\label{KKT1}
\begin{split}
\left\{
\begin{array}{llllll}
\h X^{\star}\mathrm{diag}(\h s^{\star})=\h M^{\star},\\
\nabla \h x_j^{\star} =\h u_j^{\star}, \text{ for } j= 1,\dots, L,\\
\h W^{\star} =\h X^{\star},\\
\h t^{\star} =\h s^{\star}.\\
\end{array}
\right.
\end{split}
\end{equation}
The optimality condition associated with $\h X$-subproblem can be written as:
\begin{equation}
\begin{split}
&(\h X^{\star}\mathrm{diag}(\h s^{\star})-\h M^{\star}+\h C^{\star}) \mathrm{diag}(\h s^{\star}) +\sum_{j=1}^L\nabla^T(\nabla \h x_j^{\star}-\h u_j^{\star}\\
&+\h d_j^{\star})+(\h X^{\star}-\h W^{\star}+\h E^{\star})=\h 0.\\
\end{split}
\end{equation}
Using \eqref{KKT1}, we have:
\begin{equation}
\h C^{\star}\mathrm{diag}(\h s^{\star})  + \sum_{j=1}^L\nabla^T \h d_j^{\star}+\h E^{\star}=\h 0.
\end{equation}
Similarly, by the optimality conditions associated with $\h s$ and $\h M$-subproblems, we obtain
\begin{equation}
\begin{split}
\left \{
\begin{array}{lllll}
\h X^{\star T}\h C^{\star}+\rho \mathrm{diag}(\h g^{\star})&=\h 0,\\
\h A^T(\h Y-\h A\h M^{\star})-\rho \h C^{\star}&=\h 0.\\
\end{array}
\right.
\end{split}
\end{equation}
The proof of the final relationship on the stationary condition  in \eqref{eq10} can be found in~\cite{wen2012alternating}, the specific proof of the relationship is omitted here.
Lastly, from optimality condition in the $\h W$-subproblem  in \eqref{eq8} we get
\begin{equation*}
\begin{split}
& \langle-\h E^k, \h W-\h W^{k+1}\rangle\geq  -\langle \h W^{k+1}- \h X^{k+1}, \h W-\h W^{k+1}\rangle   \\ 
& \geq  -\| \h W^{k+1}- \h X^{k+1}\| \|\h W-\h W^{k+1}\|, \  \forall \ \h W \in \Omega_1,\\
& \langle -\h g^k, \h t-\h t^{k+1}\rangle\geq   -\langle \h t^{k+1}-\h s^{k+1}, \h t-\h t^{k+1}\rangle \\ &  \geq -\| \h t^{k+1}-\h s^{k+1}\| \|\h t-\h t^{k+1}\|,  \forall \ \h g \in \Omega_2.
\end{split}
\end{equation*}
Thus, we have
\begin{equation*}
\begin{split}
&\langle\h E^{\star},\h W-\h W^{\star}\rangle \leq 0 \quad \forall \ \h W\geq \h 0,\\
&\langle\h g^{\star},\h t-\h t^{\star}\rangle \leq 0 \quad \forall \ \h g\geq \h 0.
\end{split}
\end{equation*}
\end{proof}
\begin{remark}
    Note that  \eqref{equ1} is a non-convex optimization problem with respect to both $\h X$ and $\h s$, and our algorithm \eqref{eq8} is a four-block ADMM. The corresponding convergence analysis is very challenging. We refer to some existing work \cite{chang2016phase,bui2023stochastic,wen2012alternating} and establish the convergence under the assumption of the successive differences of the multipliers, which is empirically verified in \Cref{sec:results}. 
\end{remark}

\section{Experiments and Results}\label{sec:results}
In this section, we will evaluate the performance of the proposed methods quantitatively and visually on both synthetic and actual datasets.
Regarding the comparison with different priors,
our methods were divided into two groups:  the model-based method (RUM$_{\tiny \mbox{TV}}$) and learning-based methods (RUM$_{\tiny \mbox{PnP}_{\tiny 1}}$ denotes as  PnP with DnCNN~\cite{zhang2017beyond}, RUM$_{\tiny \mbox{PnP}_{\tiny 2}}$  denotes as PnP with BM3D~\cite{dabov2007image}). These proposed methods will be assessed in comparison to the traditional methods, namely edge-50 and LCF.
\subsection{Experimental Settings and Evaluation Metrics}
\begin{figure}[htbp]
		\begin{center}
  \begin{tabular}{cc}
     \includegraphics[width=0.45\textwidth]{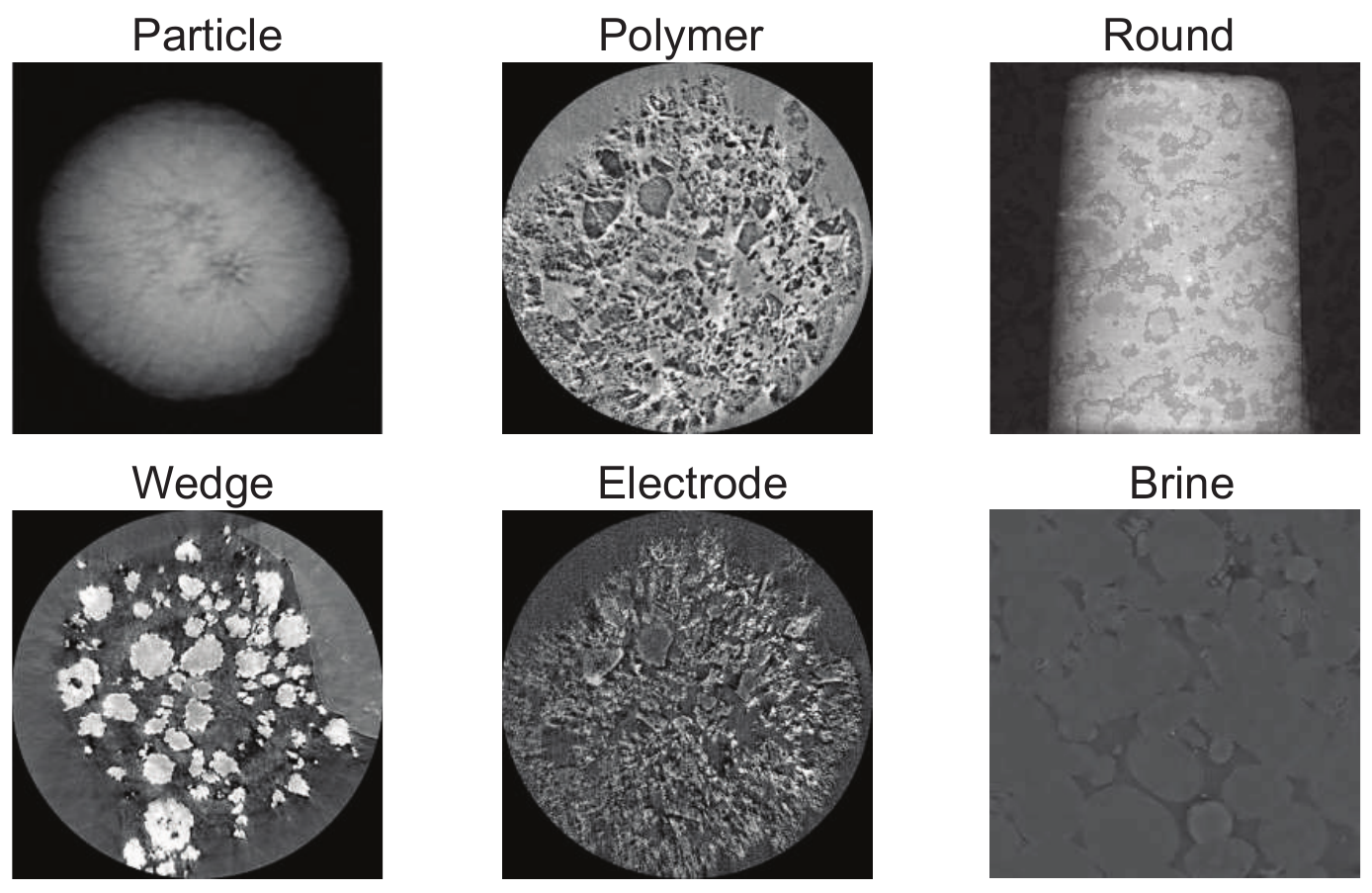} 
       \end{tabular}
		\end{center}
		\caption{Typical examples of the
test datasets: projections (top) and reconstructed slices (bottom).}
          \label{real_data}
	\end{figure}
\begin{figure}[htbp]
		\begin{center}
  \begin{tabular}{cc}
     \includegraphics[width=0.30\textwidth]{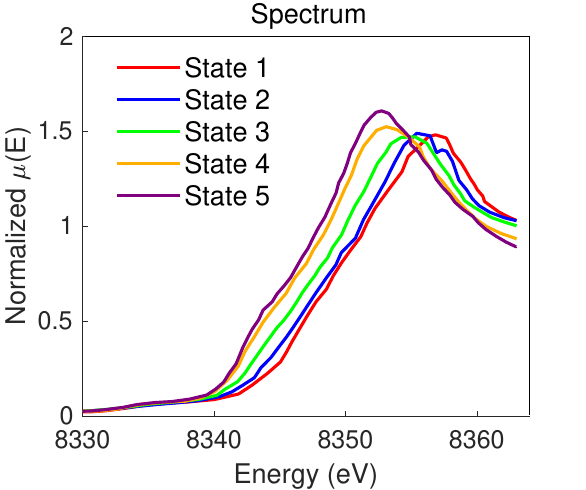} 
       \end{tabular}
		\end{center}
		\caption{Normalized spectra under different Ni valence states of X-ray XANES in a battery cathode. State 1, 2, 3, 4, 5 represent the different valence states of Ni, respectively.}
          \label{spectrum}
	\end{figure}
 
 \textbf{Data Description}. The dataset presented in Fig.  \ref{real_data} comprises three X-ray projection images (Particles, Polymer and Round) and three reconstructed slices (Wedge, Electrode and Brine), which were utilized to create a simulation of 2D and 3D TXM-XANES imaging scenarios. 
 To generate simulated movie data, as shown in Fig.  \ref{spectrum}, 
 the reference spectra of different Ni valence states were randomly assigned to pixels in the images for various phase maps. 
The sample is assumed to contain various valence states of Ni elements, and the proportion of Ni elements satisfies the sum-to-one constraint. We use number $(1,2,3,\dots, L) $ to describe the state. 
    
    \textbf{Evaluation Metrics}.
 Each synthetic dataset frame is further corrupted with additive Gaussian noise with varying noise levels, with the standard deviation $\sigma \in [1, 7]$.
For the performance assessment of the algorithms, we utilize two commonly used criteria to measure the accuracy of the phase map: the peak signal-to-noise ratio (PSNR) and the structural similarity index (SSIM). 
PSNR is defined as follows:
\begin{equation}
    \text{PSNR}=20\times \log _{10} \big (\text{MAX}/\text{RMSE} ),
\end{equation}
where $\text{MAX}$ is the maximum pixel value of the estimated image $\hat{\h X}$, and $\text{RMSE}$ is the root mean square error between $\hat{\h X}$ and the ground truth $\h X$. The RMSE is defined as:
\begin{equation}\label{eq:rmse}
    \text{RMSE}=\sqrt{\frac{1}{n_1n_2}\sum_{i=1}^{n_1} \sum_{j=1}^{n_2} \|\hat x(i,j)-x(i,j)\|^2},
\end{equation}
where $n_1$ and $n_2$ are the number of rows and columns in the image $\h X$. We use the estimated phase map $\hat{\h X}$ and the ground truth $\h X$ to calculate PSNR. 

SSIM is a metric that quantifies the similarity between two images. The SSIM formula is expressed as follows:
\begin{equation}
    \text{SSIM}(\hat{\h X},\h X)=\frac{[(2\mu_{\hat {\h X}}\mu_{\h X}+c_1)*(2\sigma_{\hat{\h X}\h X}+c_2)]}{[(\mu_{\hat {\h X}}^2+\mu_{{\h X}}^2+c_1)*(\sigma_{\hat {\h X}}^2+\sigma_{\h X}^2+c_2)]},
\end{equation}
where $\mu_{\hat {\h X}}$ and $\mu_{\h X}$ represent the means of $\hat{\h X}$ and $\h X $, respectively. $\sigma_{\hat {\h X}}$ and $\sigma_{\h X}$ denote the standard deviations of $\hat{\h X}$ and $\h X $, respectively. $\sigma_{\hat{\h X}\h X}$ is the covariance of $\hat{\h X}$ and $\h X $, while $c_1$ and $c_2$ are small constants added to prevent division by zero errors and stabilize the formula. 

   \textbf{Parameter Settings}. The maximum iteration of the RUM algorithm is set as 100 for all scenarios. To achieve the best RMSE results for simulated data,  $\lambda$ and $\rho$ are fine-tuned for various methods. 
   Fig.  \ref{parameter_correlation} illustrates the effect of parameters $\rho$ and $\lambda$ on the performance of RUM${\tiny \mbox{TV}}$ using Particle data with $\sigma=3$. 
   We observe  that $\lambda$ regulates the influence of the regularization term and significantly affects the unmixing results, whereas $\rho$ is a penalty parameter in the augmented Lagrangian function and only affects the convergence speed. 
\begin{figure}[htbp]
		\begin{center}
				\includegraphics[width=0.32\textwidth]{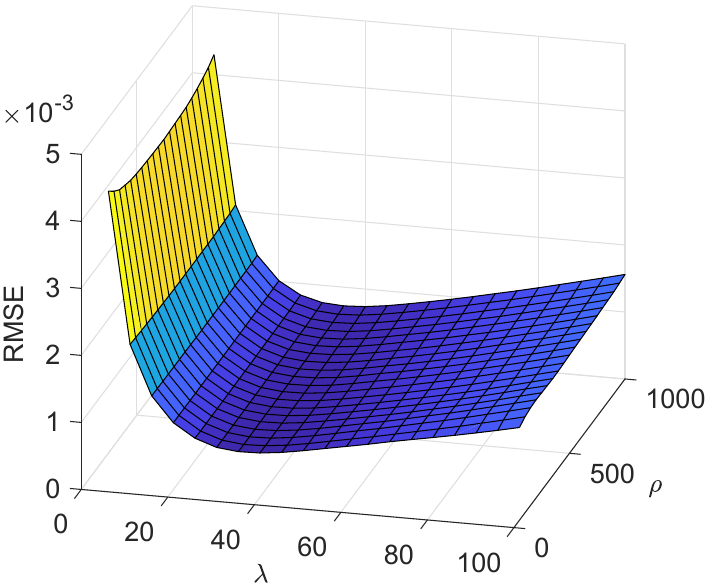} 
		\end{center}
		\caption{The study analyzed the relationship between RMSE and regularization parameters using Particle data ($\sigma=3$) and evaluated the performance of the RUM$_{\tiny \mbox{TV}}$  method.}
  \label{parameter_correlation}
\end{figure}
\subsection{Analysis of the proposed algorithm}
 \begin{figure}[htbp]
		\begin{center}
                    \includegraphics[width=0.24\textwidth, height=0.2\textwidth]
                    {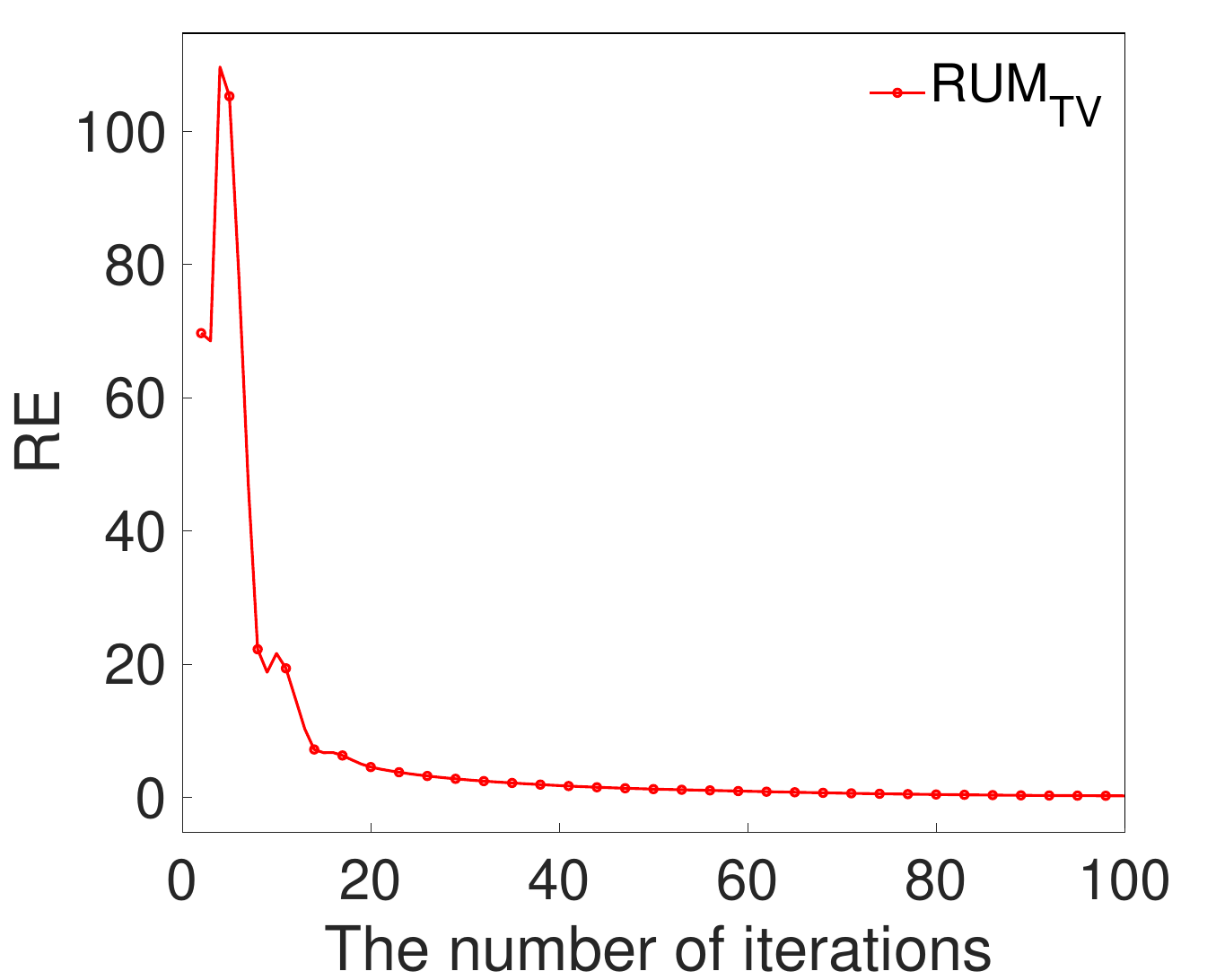}
                    \hspace{-10pt}
                    \includegraphics[width=0.24\textwidth, height=0.201\textwidth]
                    {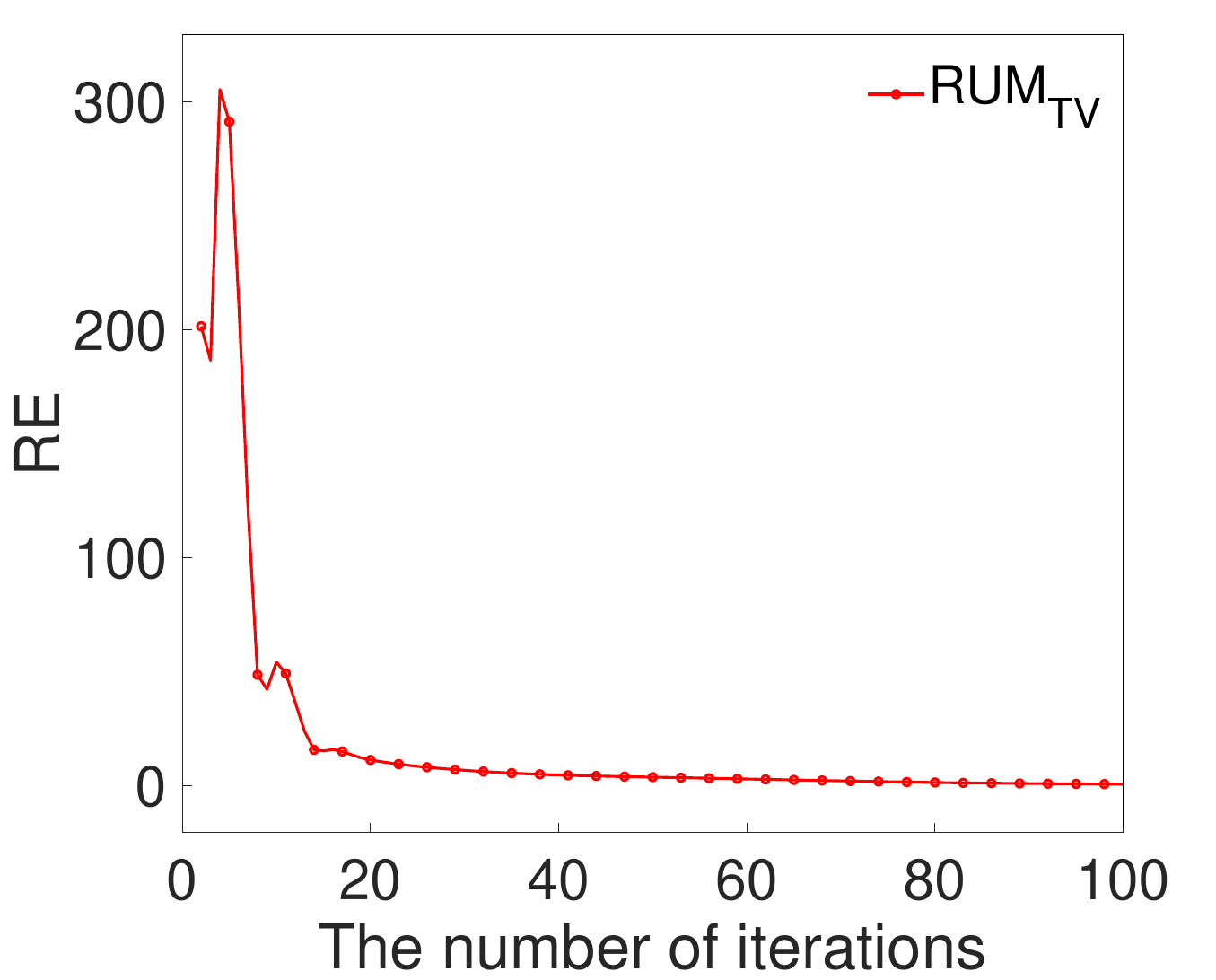}\\
                    			    \includegraphics[width=0.24\textwidth, height=0.2\textwidth]{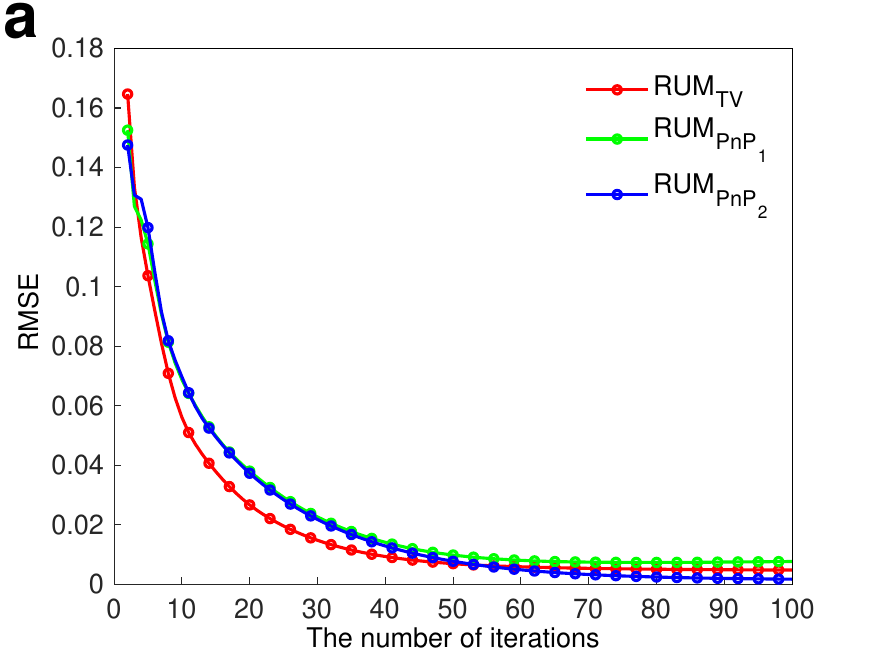}
        \hspace{-10pt}
                    \includegraphics[width=0.24\textwidth, height=0.2\textwidth]{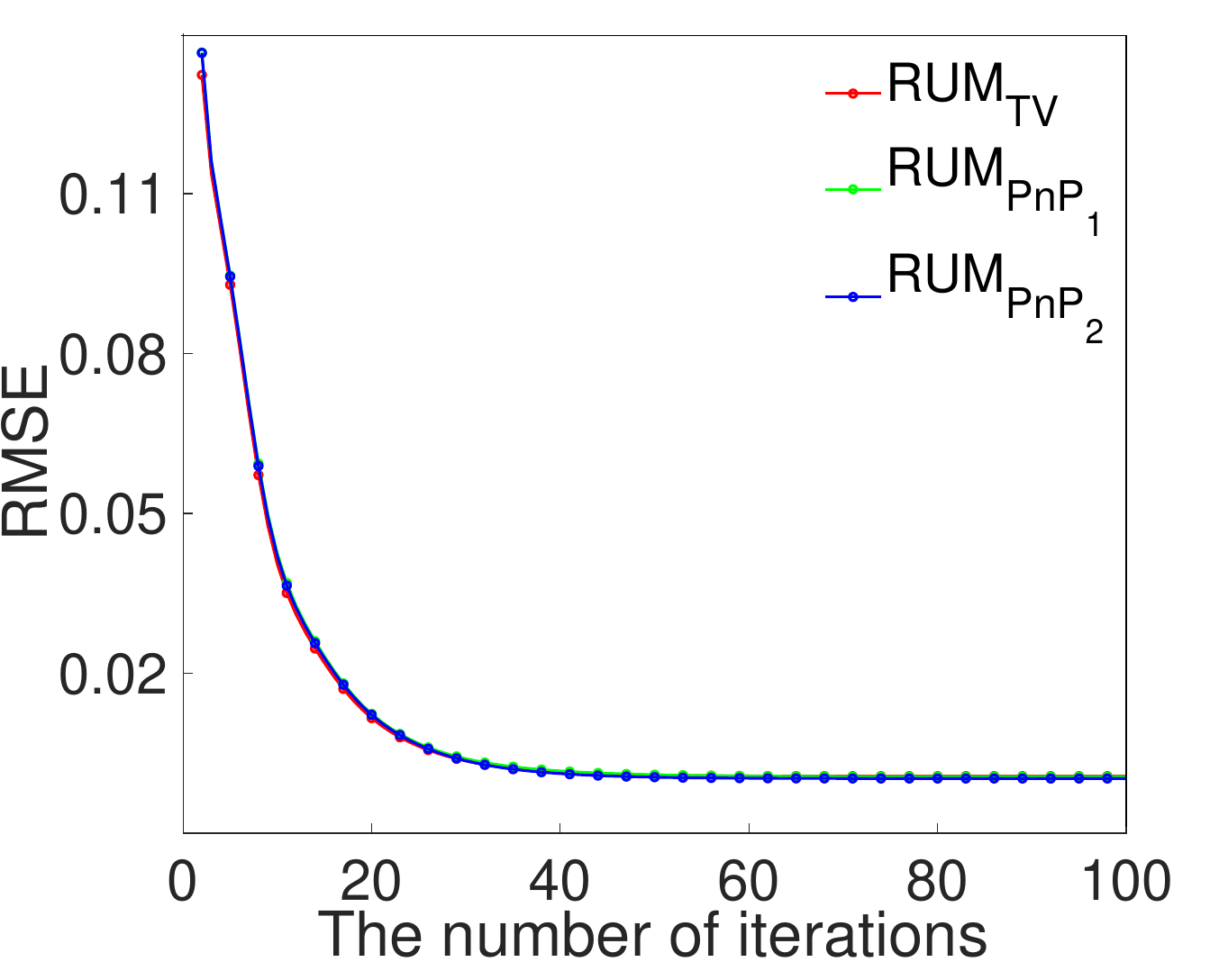}
		\end{center}
		\caption{The RE and RMSE curves of our proposed methods were analyzed using the Particle (left) data and the Brine data (right) under the noise level  $\sigma=3$. }
  \label{convergence}
	\end{figure}
\textbf{Convergence}.  In \Cref{sec:convergence}, we have proved the convergence analysis on  Algorithm \ref{alg1} with the TV regularization under some assumptions. 
Here we conduct experiments to empirically verify the assumption and demonstrate the convergence outcomes of three proposed methods with both TV and PnP priors. For each iteration, we  plot the relative errors (RE) defined as $\|\h C^{k+1}-\h C^k\|_F^2+\|\h D^{k+1}-\h D^k\|_F^2+\|\h E^{k+1}-\h E^k\|_F^2+\|\h g^{k+1}-\h g^k\|_F^2 $
 to verify the assumption in \Cref{thm1}.  In addition, we also plot the RMSE in \eqref{eq:rmse}
 with respect to the iteration. 
Fig.  \ref{convergence} shows that the value of RE goes to zero during the iteration, which is consistent with the assumption in \Cref{thm1}. 
Three methods achieve stable RMSE values after 100 outer iterations on both datasets. 
Hence we set the maximum iteration number as 100 in the following experiment to reduce the computation time. 
\begin{table}[htbp]
\scriptsize
\caption{Comparison of the computational time (units: seconds). Particle and Brine data have the image size of $379 \times 520 \times 969$ and $761 \times 742 \times 969$, respectively. }
\begin{center}
\begin{tabular}{lcccccc}
\hline
   &Edge-50 &  LCF & RUM$_{\tiny \mbox{TV}}$ & RUM$_{\tiny \mbox{PnP}_{\tiny 1}}$  &  RUM$_{\tiny \mbox{PnP}_{\tiny 2}}$ \\
\noalign{\smallskip}\hline\noalign{\smallskip}	
 Particle    & \textbf{2.40} &   4.11 &   9.78  &  13.51	 &    210.54  \\ 
Brine & \textbf{10.97} &  31.40  & 42.20   &  46.72     &  333.59  \\ \hline    
\end{tabular}
\end{center}
\label{computational_time}
\end{table}

 \textbf{Running Time}. In order to find a balance between performance improvement and computational efficiency, we conducted experiments to evaluate the running times of various datasets on the CPU. 
 It is worth noting that all of these experiments were conducted on a computer equipped with an Intel i7-6348 2.6 GHz CPU and 16 GB of RAM.
 Table \ref{computational_time} shows the average computation time for the Particle and Brine datasets. The edge-50 method requires the least amount of time because it does not require any iterative computation. RUM$_{\tiny \mbox{TV}}$ is much smaller than the other priors. In particular, when the data size is larger,  the running time of RUM with DnCNN is much smaller than BM3D. Our observations indicate that the running time of the proposed framework depends mainly on the complexity of the different priors. 

\textbf{Selection of the PnP prior}.
The proposed PnP framework offers great flexibility, allowing us to incorporate various state-of-the-art image-denoising approaches. This study examines two notable denoisers: DnCNN~\cite{zhang2017beyond}  and BM3D~\cite{dabov2007image}. The results in Table \ref{compare_pnp} demonstrate that both methods significantly enhance the PSNRs and SSIMs across different datasets. However, when observing Fig.  \ref{particle_pnp}, it becomes evident that RUM with BM3D achieves superior outcomes. Therefore, we employ BM3D as our chosen denoising approach for the following experiments, referred to as RUM$_{\tiny \mbox{PnP}}$.
\begin{table}[htbp]
\scriptsize
\caption{Comparison of PSNR (dB) and SSIM for two datasets using different PnP priors ($\sigma=3$).}
\begin{center}
\begin{tabular}{lcccc}
\hline
    & \multicolumn{2}{c}{Particle} & \multicolumn{2}{c}{Brine} \\
    \cmidrule(lr){2-3}\cmidrule(lr){4-5}
    & PSNR& SSIM & PSNR&SSIM \\
\noalign{\smallskip}\hline\noalign{\smallskip}	
 $\mbox{PnP}_1$   &  33.20  & 0.84      & 34.51 & 0.78\\
 $\mbox{PnP}_2$ &   \textbf{37.36}  &  \textbf{0.96}     & \textbf{39.85} & \textbf{0.94} \\ \hline 
\end{tabular}
\end{center}
\label{compare_pnp}
\end{table}
\begin{figure}[htbp]
		\begin{center}
				\includegraphics[width=0.47\textwidth]{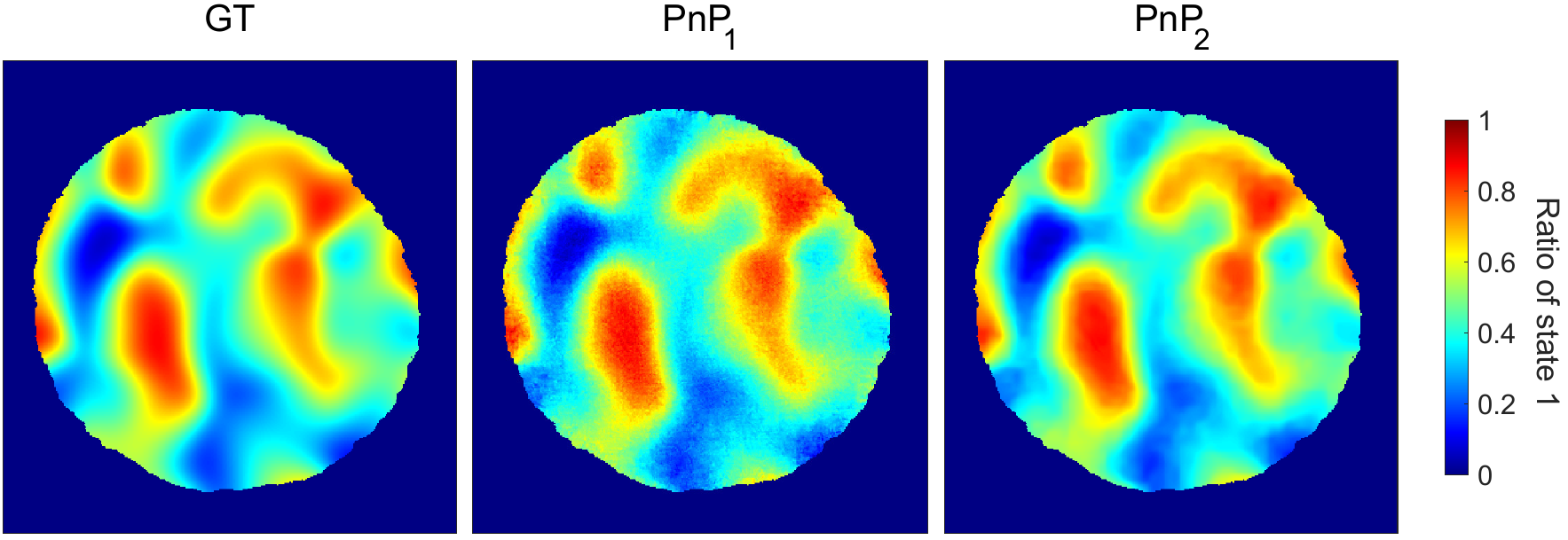} \\
    \includegraphics[width=0.47\textwidth]{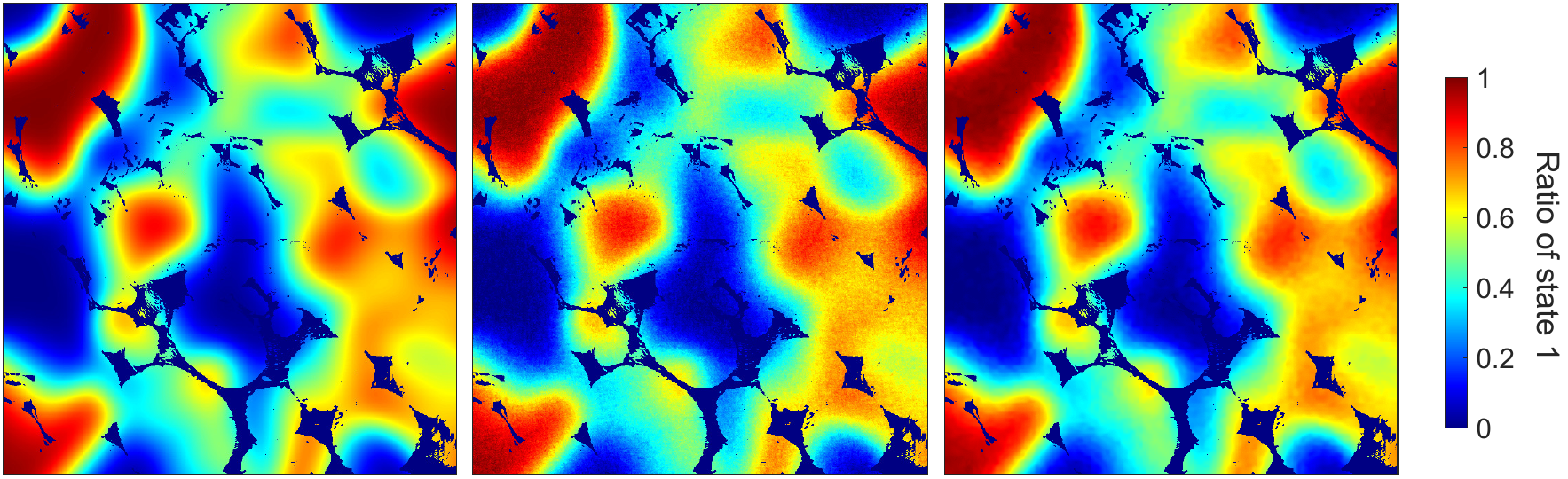}
		\end{center}
		\caption{A visual comparison of the chemical phase map for the different PnP priors on the Particle data (top) and the Brine data (bottom) when $\sigma=3$. }
  \label{particle_pnp}
\end{figure}
\subsection{Results of synthetic datasets}
\begin{figure*}[htbp]
       \centering
       \includegraphics[width=0.95\textwidth]{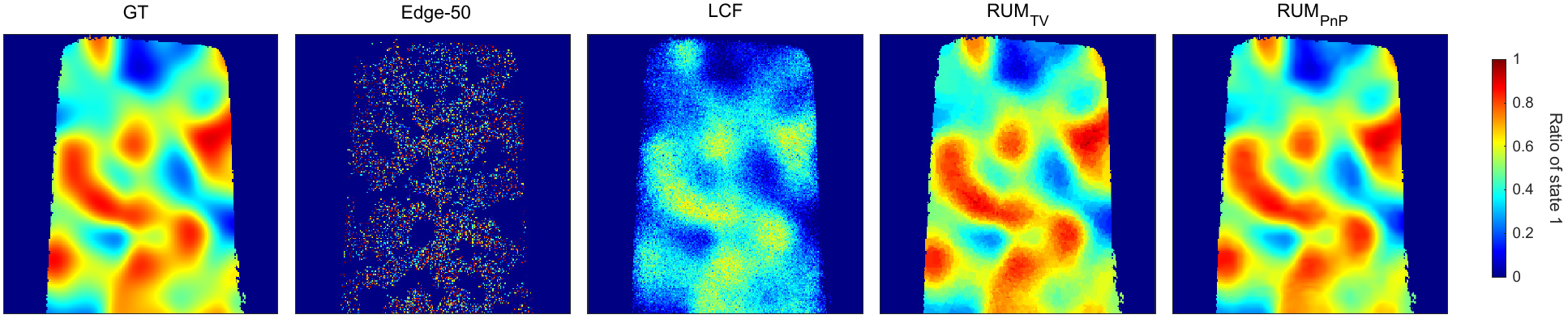} 
       \\
       \includegraphics[width=0.951\textwidth]{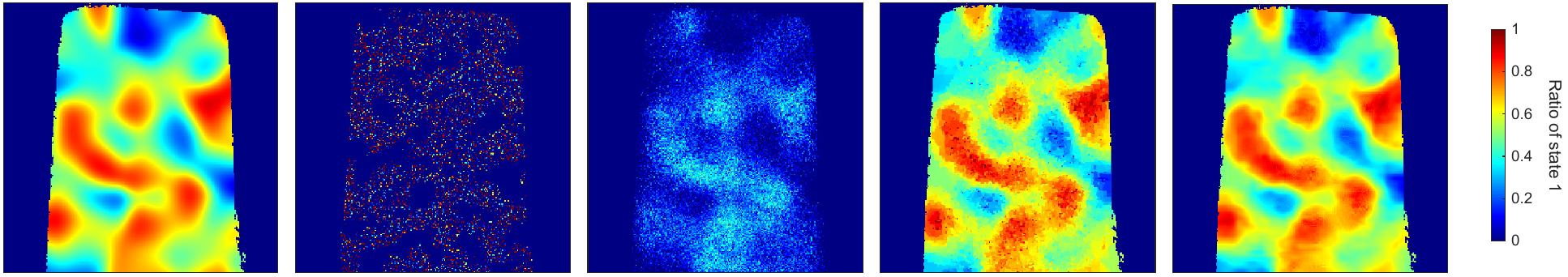}
       \caption{A visual comparison of the chemical phase map for various methods on the Round data under different noise levels ($\sigma=3$ on the top and $\sigma=7$ on the bottom). Note that the other chemical map is the reverse since $L=2$.}
        \label{2spectrum_3sigma}
 \end{figure*} 

 \begin{table*}[htbp]
\caption{Comparison of PSNR (dB) and SSIM for simulated datasets using different approaches and noise levels.}
\begin{center}
\begin{tabular}{lccccccccc}
\hline
 \multirow{2.5}{*}{Test set} & \multirow{2.5}{*}{$\sigma$} & \multicolumn{2}{c}{Edge-50} &  \multicolumn{2}{c}{LCF}  & \multicolumn{2}{c}{RUM$_{\tiny \mbox{TV}}$}   &  \multicolumn{2}{c}{RUM$_{\tiny \mbox{PnP}}$ } \\
\cmidrule(lr){3-4}\cmidrule(lr){5-6}\cmidrule(lr){7-8}\cmidrule(lr){9-10}
  &        & PSNR& SSIM & PSNR&SSIM & PSNR&SSIM &  PSNR&SSIM\\
\noalign{\smallskip}\hline\noalign{\smallskip}	
\multirow{4}{*}{Particle}  &  1 & 8.98& 0.10 	 & 16.11& 0.57   &   38.22&0.94     & \textbf{42.65}&\textbf{0.98}   \\
        & 3  & 5.96 & 0.08   & 9.33&0.28  &  32.64& 0.87   &  \textbf{37.36}& \textbf{0.96}  \\
        & 5   & 5.73&0.07  & 6.80&0.22 & 29.59&0.76   & \textbf{34.39}& \textbf{0.93} \\
  	  & 7  & 5.71&0.07  & 5.26&0.20  &    27.63&0.69  & \textbf{31.39}& \textbf{0.87}   \\\hline
\multirow{4}{*}{Electrode} & 1     &  9.19 &0.23   & 17.37&0.51  &   42.39&0.95   &  \textbf{46.79}& \textbf{0.99}   \\
          & 3    & 6.23&0.22    & 10.00&0.36 & 38.38&0.92   & \textbf{42.27}& \textbf{0.98}   \\
          & 5  & 6.01&0.22   & 7.29&0.35&  33.65&0.83  & \textbf{37.97}&\textbf{0.95}     \\
  	  & 7   &  5.97&0.21  & 6.03&0.36 &  31.51&0.78  & \textbf{35.21} &\textbf{0.92}     \\\hline
 \multirow{4}{*}{Polymer}  & 1   & 11.20&0.23   &  19.94&0.59 &  42.36&0.95   &  \textbf{47.10} & \textbf{0.99}   \\
            & 3   & 6.56&0.22  & 11.86&0.39 &  37.85&0.89 &   \textbf{42.74}& \textbf{0.98}    \\ 
            & 5     & 6.07&0.21 & 8.94&0.35 &  32.44&0.75    &  \textbf{38.53} & \textbf{0.94}   \\ 
     	& 7   & 5.97&0.21 &  7.27&0.35  &  30.00&0.68  &  \textbf{34.29}& \textbf{0.86}  \\ \hline  
 \multirow{4}{*}{Wedge}  & 1   & 10.04&0.23  & 19.23&0.56  &  47.26&0.99  & \textbf{51.01} & \textbf{1.00}  \\
            & 3   & 6.58&0.21    & 10.97&0.35 &  38.98&0.93   & \textbf{43.35} &\textbf{0.99}    \\ 
            & 5    & 6.23&0.21  & 8.15&0.34  &  33.65&0.83   & \textbf{38.14}&\textbf{0.95}    \\ 
     	& 7   & 6.13&0.21 &  6.88&0.32 &  31.41&0.76  &  \textbf{34.86}&\textbf{0.89}  \\ \hline   
\multirow{4}{*}{Round}  & 1   & 7.06&0.04   & 12.17&0.50  &  35.28&0.92 &  \textbf{39.58} & \textbf{0.97}   \\
            & 3   & 4.89&0.01 &   6.81&0.28  & 28.53&0.80  &  \textbf{34.14}& \textbf{0.94}   \\ 
            & 5     & 4.66&0.01   &  5.06&0.22 & 26.35&0.70  & \textbf{30.78}&\textbf{0.88}  \\ 
     	& 7   & 4.63&0.01  &  3.95&0.19  &   24.43&0.63   &   \textbf{28.19} &\textbf{0.80} \\ \hline 
\multirow{4}{*}{Brine}  & 1   & 10.46&0.03  &  19.86&0.52 &   40.34&0.92    & \textbf{45.40}  & \textbf{0.97}  \\
            & 3   & 4.96&0.01  & 10.54&0.24 & 35.59&0.83   &  \textbf{39.85} & \textbf{0.94}  \\ 
            & 5    & 4.59&0.01  &  7.08&0.18 & 30.90&0.66  &  \textbf{37.01} & \textbf{0.91}  \\ 
     	& 7    & 4.55&0.01  &   5.17&0.17  &  28.52&0.56 & \textbf{33.00} & \textbf{0.80}  \\ \hline 
\end{tabular}
\end{center}
\label{differ_noise_result}
\end{table*}

\textbf{Different Noise Levels}. Table \ref{differ_noise_result} displays the performance of both traditional methods and our proposed unmixing methods incorporating TV and PnP priors with two reference spectra. The optimal results are highlighted in bold font. Overall, our two methods all outperform the traditional techniques for all the datasets. They exhibit remarkable robustness to a wide range of noise levels, particularly when the noise is substantial, as the chemical map is still reconstructed effectively. RUM$_{\tiny \mbox{PnP}}$ performs best under all noise conditions.
In Fig.  \ref{2spectrum_3sigma}, we compare the chemical phase maps of Round data obtained from various approaches under two kinds of noise levels. Except for edge-50, we observe that the estimated phase maps are consistent with the ground truth. However, under strong noise, our proposed methods yield less noisy phase maps closer to the ground truth and preserve the image details. 
The edge-50 and LCF methods are pixel-based and do not consider the spatial-spectral correlations in XANES images. Our proposed RUM$_{\tiny \mbox{TV}}$ method employs fixed regularizers and lacks flexibility, while the PnP framework models priors using denoiser, thus eliminating the need for handcrafted regularizers.
\begin{table}[htbp]
\scriptsize
\caption{Comparison of PSNR (dB) and SSIM using different approaches with varying numbers of the reference spectra ($\sigma=3$).}
\begin{center}
\begin{tabular}{lccccccccc}
\hline
 \multirow{2.5}{*}{Test set} & \multirow{2.5}{*}{$L$}  &  \multicolumn{2}{c}{LCF} & \multicolumn{2}{c}{RUM$_{\tiny \mbox{TV}}$}  &  \multicolumn{2}{c}{RUM$_{\tiny \mbox{PnP}}$ }\\
\cmidrule(lr){3-4}\cmidrule(lr){5-6}\cmidrule(lr){7-8}
  &        & PSNR& SSIM & PSNR&SSIM & PSNR&SSIM \\
\noalign{\smallskip}\hline\noalign{\smallskip}	
 \multirow{3}{*}{Particle} &3  & 13.46&0.52& 22.56 &0.80 & \textbf{22.64}&  \textbf{0.92} \\
     &4   & 13.36&0.48 & 20.32 &0.73  & \textbf{21.23}  & \textbf{0.85} \\ 
     & 5  &  13.72 &0.48    & \textbf{21.31} &0.76 &  19.45 &\textbf{0.89}     \\ \hline
\multirow{3}{*}{Wedge}     &3  & 11.47&0.27 & 24.55& 0.74   &  \textbf{26.23} &\textbf{0.92}    \\
     &  4 & 10.66 &0.24 & 18.05&0.61  &  \textbf{19.34}&\textbf{0.75}  \\ 
     & 5 & 12.62 &0.25  & 19.26&0.66 & \textbf{20.13} & \textbf{0.80}\\ \hline  
\end{tabular}
\end{center}
\label{differ_spectrum_result}
\end{table}

\begin{figure}[htbp]
		\begin{center}
				\includegraphics[width=0.45\textwidth]{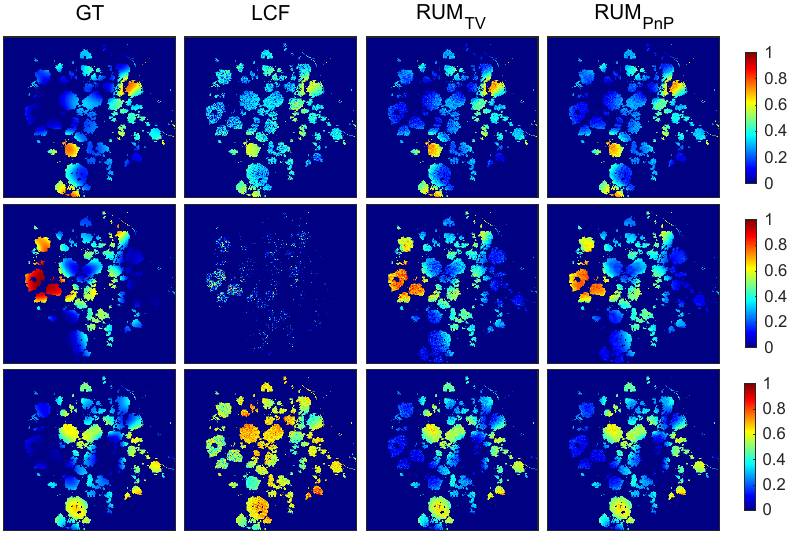} 
		\end{center}
		\caption{The visual comparison for the chemical phase maps of the various methods on Wedge data under three reference spectra ($\sigma=3$). From top to bottom: Ni valence state 1. 2, 3, respectively.}
   \label{3spectrum_3sigma}
	\end{figure}

	\begin{figure}[htbp]
		\begin{center}
				\includegraphics[width=0.45\textwidth]{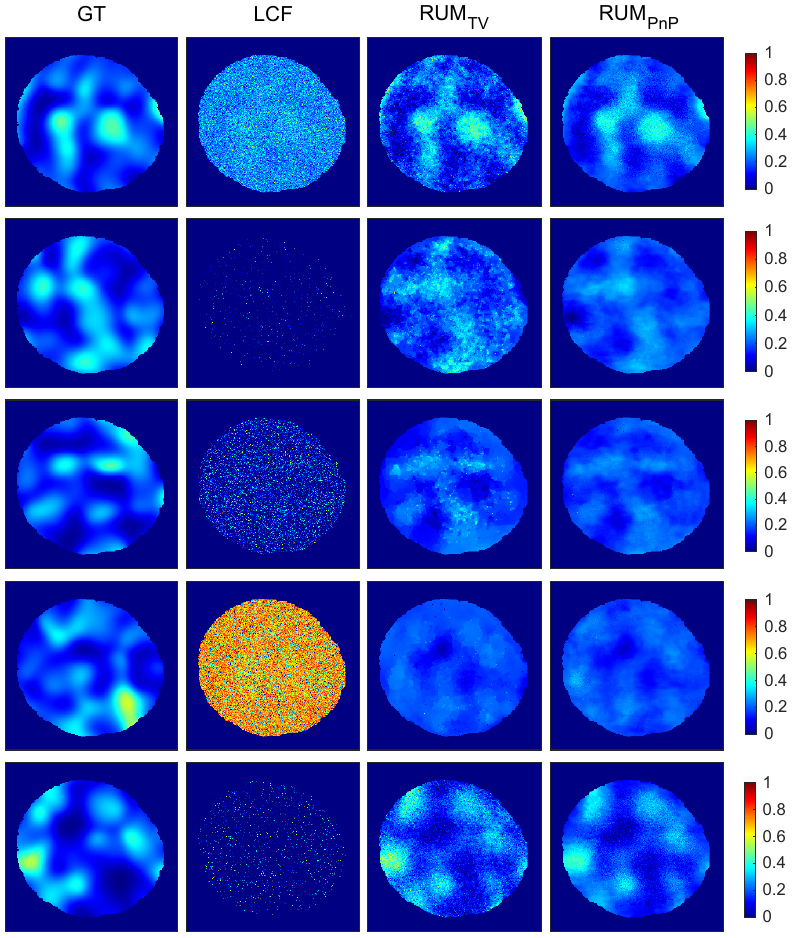} 
		\end{center}
		\caption{The visual comparison for the chemical phase maps of the various methods on Particle data under five reference spectra ($\sigma=3$). From top to bottom: Ni valence state 1, 2, 3, 4, 5, respectively.}
           \label{5spectrum_3sigma}
	\end{figure}

\textbf{Number of  reference spectra}. To evaluate the capability of unmixing multiple spectra for XANES data, we generated two datasets at reference spectra ($L=3, 4, 5$) when $\sigma=3$. Here different reference spectra represent different valence states of Ni. The results of PSNRs and SSIMs are presented in Table \ref{differ_spectrum_result}, demonstrating our proposed framework's robustness and superiority. Additionally, Fig.  \ref{3spectrum_3sigma} displays the phase maps of  three reference spectra using Wedge data, indicating that our phase maps are closer to the ground truth. Furthermore, Fig.  \ref{5spectrum_3sigma} shows the result with the number of reference spectra being 5. The phase maps with Particle data for Ni valence state 1, 2, and 5 obtained RUM method exhibit clearer structural details. However, the structure of the phase map for Ni valence state 3 and 4 is not very clear, suggesting a strong correlation between its reference spectra. Nevertheless,  RUM$_{\tiny \mbox{PnP}}$ still outperforms other methods in unmixing multiple spectra.
\subsection{Results of real dataset}
\label{c44}
\begin{figure}[htbp]
    \centering
    \includegraphics[width=0.4\textwidth]{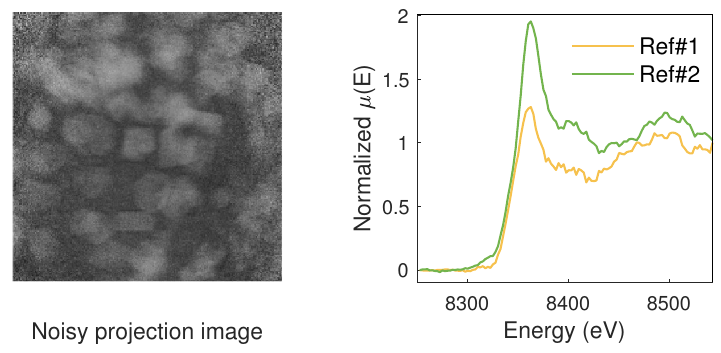} 
\caption{Low SNR projection image of TXM-XANES recording (left) and the reference spectra dictionary (right).}
    \label{projection_image}
\end{figure}
\begin{figure}[htbp]
    \centering
    \includegraphics[width=0.35\textwidth]{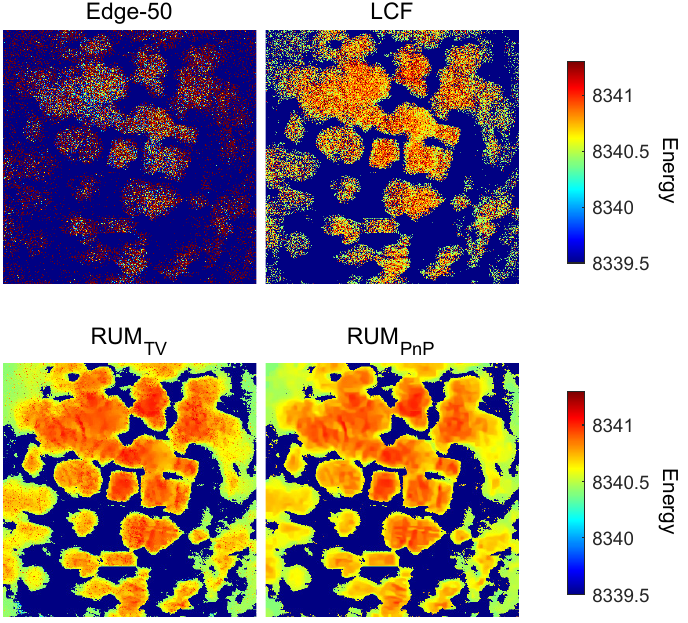} 
\caption{The chemical state maps, generated from noisy data using various methods, display the energy of the Ni valence state through their colors. Using RUM enhances the interpretability of real data and enables the identification of inter and intra-particle state differences. }
    \label{real_NCM}
\end{figure}
We apply the proposed RUM$_{\tiny \mbox{TV}}$ and RUM$_{\tiny \mbox{PnP}}$ methods to unmix actual real TXM-XANES data. The data comprises an image of numerous NCM particles on a charged cathode, as shown on the left in Fig.  \ref{projection_image}. The exposure time for a single frame was 0.5 seconds, and the data was captured at 117 energy points from 8180 eV to 8562 eV.

The NCM particle data exhibits an extremely low signal-to-noise ratio, making it challenging to discern the reference spectra of Ni elements in the range of 8180 eV to 8562 eV under practical conditions. Consequently, we can only determine that it contains Ni at different internal states, similar to the blind unmixing. To address this issue, we employ the denoising algorithm~\cite{li2022subspace} followed by VCA~\cite{nascimento2005vertex} to improve the signal-to-noise ratio and dictionary extraction. We focus  on two Ni states where the reference spectra of these states extracted from the range of 8180 eV to 8562 eV are illustrated on the right in Fig.  \ref{projection_image}. These techniques allowed us to overcome the low signal-to-noise ratio and extract valuable information from the NCM particle data. 

As shown in Fig.  \ref{real_NCM}, the proposed RUM algorithms  clearly distribute chemical elements in the NCM particle structure. Note that RUM$_{\tiny \mbox{TV}}$ has some theoretical guarantee on the convergence while the RUM$_{\tiny \mbox{PnP}}$ show better unmixing results in the synthetic experiments. On the other hand, due to the high noise levels in each projection image of NCM particles, the chemical map obtained using edge-50 and LCF fail to provide any meaningful information. 
Our methods simultaneously unmix and denoise the chemical imaging data, avoiding the accumulated error if we split these two processes. 
Additionally, the chemical phase map of NCM particles indicates an uneven reaction of the battery electrode, with some particles exhibiting a higher Ni valence state and others showing a lower Ni valence state.   
The utilization of the RUM unmixing method opens up avenues for enhanced understanding of spatiotemporally electrochemical reactions, enabling more profound insights and facilitating the optimization of composite electrode designs.

\section{Conclusion}\label{sec:concl}
This paper introduced a robust spectra unmixing framework to extract the chemical phase map signal for the widely-used X-ray imaging technique. Our proposed framework considered variance in spectra and maximized the use of spatial-spectral priors in X-ray microspectroscopy. It outperforms traditional methods when dealing with strong noise and spectral variability. Additionally, the framework exhibits favorable convergence properties for TV regularization, while the PnP prior performs better. 
Our future study includes the theoretical analysis of the proposed framework with the PnP prior.

\bibliographystyle{IEEEtranN}
\bibliography{IEEEabrv,IEEEexample}
 

\end{document}